\documentclass[letterpaper]{article}
\usepackage{amsmath,amssymb,graphicx}
\usepackage{aaai20}  
\usepackage{times}  
\usepackage{helvet}  
\usepackage{courier}  
\usepackage{url}  
\usepackage{graphicx}  
\usepackage{color}
\begin{document}
%

\newtheorem{theorem}{Theorem}
\newtheorem{lemma}{Lemma}
\newtheorem{claim}{Claim}
\newtheorem{proposition}{Proposition}
\newtheorem{definition}{Definition}
\newtheorem{corollary}{Corollary}
\renewcommand{\phi}{\varphi}
\renewcommand{\epsilon}{\varepsilon}
\newcommand{\<}{\langle}
\renewcommand{\>}{\rangle}
\newenvironment{proof}{\noindent{\sc Proof.}}{\hfill $\boxtimes\hspace{2mm}$\linebreak}
\newcommand{\qed}{\hfill $\boxtimes\hspace{1mm}$}

\newenvironment{proof-of-claim}{\noindent{\sc Proof of Claim.}}{\hfill $\boxtimes\hspace{2mm}$\linebreak}

\renewcommand{\H}{{\sf H}}
\newcommand{\C}{{\sf C}}
\newcommand{\N}{{\sf N}}
\newcommand{\B}{{\sf B}}
\newcommand{\cN}{{\sf \overline{N}}}
\newcommand{\cR}{{\sf \overline{R}}}
\newcommand{\KR}{\,\mbox{\scalebox{.75}{\framebox(10,10){$\sf R$}}}\,}
\newcommand{\KN}{\,\mbox{\scalebox{.75}{\framebox(10,10){$\sf N$}}}\,}
\renewcommand{\Box}{\,\mbox{\scalebox{.75}{\framebox(10,10){$ $}}}\,}
\newcommand{\cKN}{\overline{\mbox{\scalebox{.75}{\framebox(10,10){$\sf N$}}}}}
\newcommand{\R}{{\sf R}}
\renewcommand{\S}{{\sf S}}
\newcommand{\D}{{\sf D}}
\newcommand{\A}{{\sf A}}

\newsavebox{\diamonddotsavebox}
\sbox{\diamonddotsavebox}{$\Diamond$\hspace{-1.8mm}\raisebox{0.3mm}{$\cdot$}\hspace{1mm}}
\newcommand{\diamonddot}{\usebox{\diamonddotsavebox}}

\title{Subgame-Perfect Blameworthiness}

\title{Counterfactual Reasoning in  Extensive-Form Games}

\title{Blameworthiness in Security Games}


\author{Pavel Naumov \\Department of Mathematical Sciences\\  Claremont McKenna College\\Claremont, California 91711\\pgn2@cornell.edu
\And  Jia Tao \\Department of Computer Science\\Lafayette College\\Easton, Pennsylvania 18042\\taoj@lafayette.edu}

\maketitle

\begin{abstract}
Security games are an example of a successful real-world application of game theory. The paper defines blameworthiness of  the defender and the attacker in security games using the principle of alternative possibilities and provides a sound and complete logical system for reasoning about blameworthiness in such games. Two of the axioms of this system capture the asymmetry of information in security games.
\end{abstract}


\section{Introduction}

In this paper we study the properties of blameworthiness in security games~\cite{s34}. Security games are used for 
canine airport patrol~\cite{pjmoptwpk08aamas,jtpkrto10interfaces}, 
airport passenger screening~\cite{bsst16aaai},
protecting endangered animals and fish stocks~\cite{fst15ijcai}, 
U.S. Coast Guard port patrol~\cite{sfakt18ijcai,ats16ihs},
and randomized deployment of U.S. air marshals~\cite{sfakt18ijcai}.

\begin{figure}[ht]
\vspace{-1mm}
\begin{center}
\renewcommand{\arraystretch}{1.3}
\begin{tabular}{ l | c  c }
Defender \textbackslash Attacker  & Terminal 1 & Terminal 2  \\  \hline 
 Terminal 1 & $20$ & $120$ \\
 Terminal 2 & $200$ & $16$   
\end{tabular}
\caption{Expected Human Losses in Security Game $G_1$.}\label{losses figure}
\end{center}
\vspace{-3mm}
\end{figure}

As an example, consider a security game $G_1$ in which a defender is trying to protect two terminals in an airport from an attacker. Due to limited resources, the defender can patrol only one terminal at a given time. If the defender chooses to patrol Terminal 1 and the attacker chooses to attack Terminal 2, then the human losses at Terminal 2 are estimated at 120, see Figure~\ref{losses figure}. However, if the defender chooses to patrol Terminal 2 while the attacker still chooses to attack Terminal 2, then the expected number of the human losses at Terminal 2 is only 16, see Figure~\ref{losses figure}. Generally speaking, the goal of the defender is to minimize human losses, while the goal of the attacker is to maximize them. However, the utility functions in security games usually take into account not only the human losses, but also the cost to protect and to attack the target to the defender and the attacker respectively. Such a cost has to be converted to human lives using some factor, possibly different for the defender and the attacker. In game $G_1$, we assume that the cost of defending Terminal 1 and Terminal 2 is 8 and 4 respectively, while the cost of attacking these terminals is 12 and 8 respectively, see Figure~\ref{utility figure}.  As a result, for example, if the defender chooses to patrol Terminal 1 and the attacker chooses to attack Terminal 2, then the payoff of the defender is $-120-8=-128$ and the payoff of the attacker is $120-8=112$, see Figure~\ref{utility figure}.   

\begin{figure}[ht]
\begin{center}
\renewcommand{\arraystretch}{1.3}
\begin{tabular}{ l | c  c }
Defender \textbackslash Attacker  & Terminal 1 & Terminal 2  \\ [-1ex]
                                & (cost $12$) & (cost $8$)\\ \hline
 Terminal 1 (cost $8$) & $-28,8$ & $-128,112$ \\
 Terminal 2 (cost $4$) & $-204,188$ &  $-20,8$  
\end{tabular}
\caption{Utility Functions in Security Game $G_1$.}\label{utility figure}
\end{center}
\end{figure}

In real world examples of security games, the defender usually employs mixed strategies. For example, if the defender is using a strategy $75/25$, then he will spend $75\%$ of the time in Terminal 1 and $25\%$ of the time in Terminal 2. In practice, each morning the defender might get a randomly generated timetable that specifies at which terminal the defender should be at each time slot during the day~\cite{jtpkrto10interfaces}. The distinctive feature of security games compared to strategic games is {\em the asymmetry of information} between the players: the attacker knows the strategy employed by the defender but not vice versa. For example, while planning the attack, the attacker might visit the airport multiple times and observe that the defender spends $75\%$ of the time in Terminal 1 and $25\%$ of the time in Terminal 2. Thus, the attacker will know the mixed strategy used by the defender, but she will not know the location of the defender at the moment she plans to arrive at the airport on the day of the attack.

\begin{figure}[ht]
\begin{center}
\vspace{0mm}
\scalebox{0.48}{\includegraphics{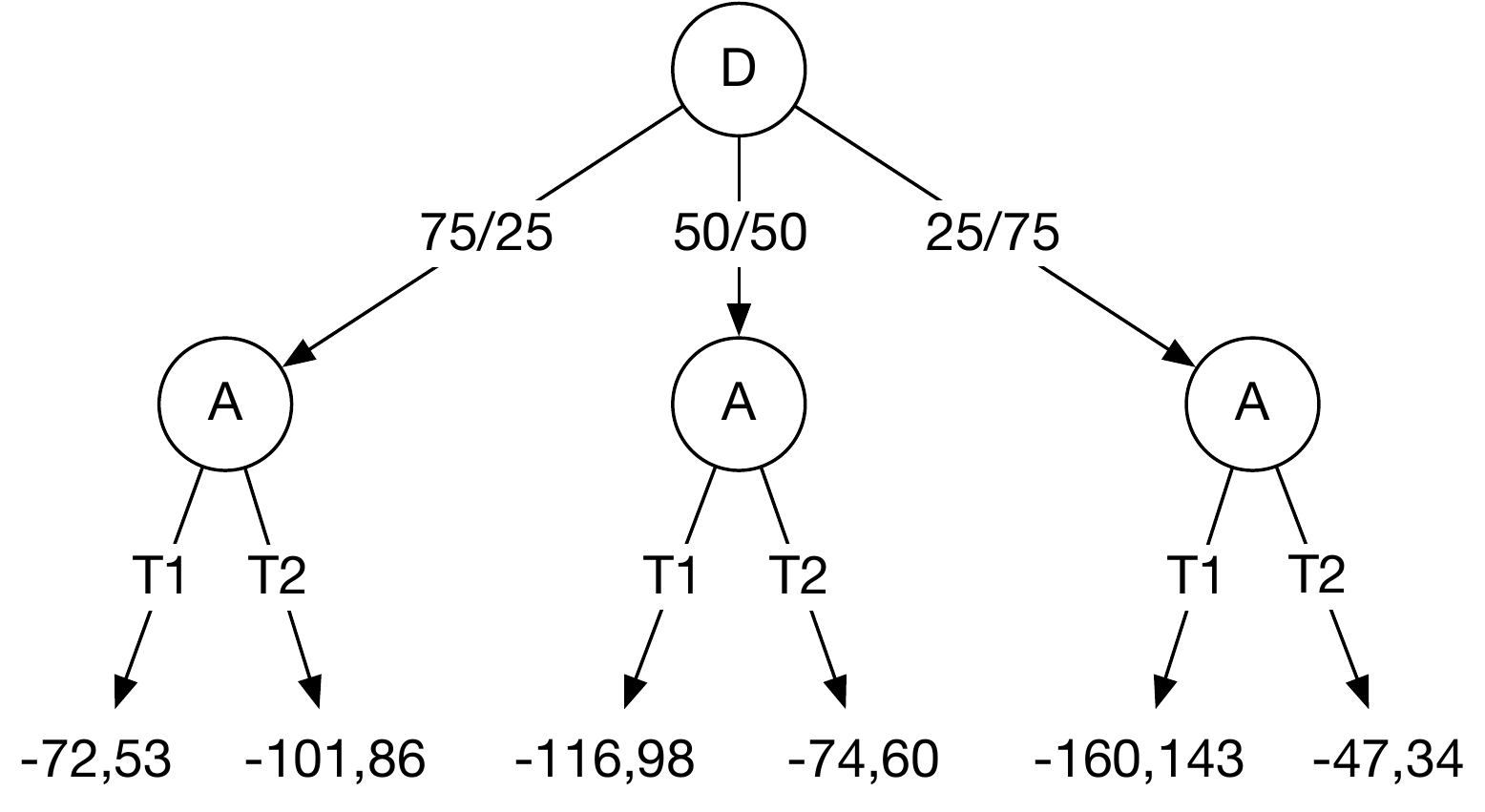}}
\vspace{0mm}
\caption{Security Game $G_1$ in Extensive Form.}\label{extensive-form-game figure}
\end{center}
\end{figure}

For the sake of simplicity, we assume that in game $G_1$ the defender must choose between only three given mixed strategies: $75/25$, $50/50$, and $25/75$. Then, game $G_1$ can be described as an extensive form game depicted in Figure~\ref{extensive-form-game figure}. The payoffs in this figure represent expected values of the utility functions. For example, suppose that the defender chooses the mixed strategy $75/25$ and the attacker chooses to attack Terminal 1. The pair $(75/25,T1)$ is called an {\em action profile} of game $G_1$. Under this action profile, the payoffs of the defender and the attacker are $-28$ and $8$, respectively, with probability $75\%$, and they are $-204$ and $188$, respectively, with probability $25\%$, see Figure~\ref{utility figure}. Thus, the {\em expected} payoff (or just ``payoff'') of the defender is
$$
75\%\times (-28)+25\%\times (-204)=-21 -51=-72
$$
and of the attacker is
$$
75\%\times 8+25\%\times 188=6+47=53.
$$
Suppose that the defender chooses a strategy $50/50$ and the attacker decides to target Terminal 2. Then, the attacker's payoff is $60$, see Figure~\ref{extensive-form-game figure}. We write this as
$$
(50/50,T2)\Vdash \mbox{``The attacker's payoff is $60$.''}.
$$

The attacker's mastermind might find this to be the attacker's fault and {\em blame} the attacker for the payoff not being at least 98. We capture the attacker's blameworthiness by
$$
(50/50,T2)\Vdash \A(\mbox{``The attacker's payoff is less than $98$.''}),
$$
where the blameworthiness modality $\A\phi$ stands for ``the attacker is blamable for $\phi$''. We define the blameworthiness using the well known Frankfurt's  principle\footnote{This principle has many limitations  that \cite{f69tjop} discusses; for example, when a person is coerced into something.} of alternative possibilities: {\em an agent is blamable for $\phi$ if $\phi$ is true and the agent could have prevented $\phi$}~\cite{f69tjop,w17}. In our case, the attacker, after learning that the  defender's strategy is 50/50, could have targeted Terminal 1, which would increase her payoff to $98$, see Figure~\ref{extensive-form-game figure}. The principle of alternative possibilities, sometimes referred to as ``counterfactual possibility''~\cite{c15cop}, is also used to define causality~\cite{lewis13,h16,bs18aaai}.

Next, assume that the defender still chooses the strategy $50/50$, but the attacker decided to target Terminal 1. Under this action profile, the payoff of the attacker is $98$, see Figure~\ref{extensive-form-game figure}. Although the payoff is less than the attacker's payoff of $143$ under the action profile $(25/75,T1)$, the attacker cannot be blamed for this:
$$
(50/50,T1)\!\Vdash\! \neg\A(\mbox{``The attacker's payoff is less than $143$.''}),
$$
because the attacker had no action in game $G_1$ to guarantee her payoff to be at least 143. At the same time, under the action profile $(25/75,T1)$, the defender is blameable for his payoff being less than $-101$:
$$
(50/50,T1)\!\Vdash\!\D(\mbox{``The defender's payoff is less than $-101$.''}),
$$
because the defender could have guaranteed his payoff to be at least $-101$ by choosing mixed strategy $75/25$, see Figure~\ref{extensive-form-game figure}. Following the principle of alternative possibilities, the blameworthiness modality $\D\phi$ stands for ``statement $\phi$ is true and the defender had a strategy to prevent it''.

In addition to the blameworthiness modalities $\A$ and $\D$, in this paper we also consider an auxiliary necessity modality $\N$. Statement $\N\phi$ stands for ``$\phi$ is true under each action profile of the given security game''. For example,
$$
(50/50,T1)\Vdash \N(\mbox{``The defender's payoff is negative.''}),
$$
because in game $G_1$ the defender's payoff is always negative. Surprisingly, as we show in Lemma~\ref{D through A lemma},  modality $\D$ can be expressed through modalities $\A$ and $\N$:
$$
\D\phi\equiv \phi\wedge \neg\N(\neg\phi\to\A\neg\phi).
$$
At the same time, we believe that modality $\A$ cannot be expressed through modalities $\D$ and $\N$, which reflects the {\em asymmetric} nature of security games.

In this paper we give a sound and complete axiomatization of the interplay between modalities $\A$ and $\N$ in security games. This work is related to our paper on blameworthiness in strategic games~(\citeyear{nt19aaai}). 
They proposed a sound and complete axiomatization of the interplay between the necessity modality $\N$ and the coalition blameworthiness modality $\B_C$ in strategic games. Their definition of the blameworthiness is also based on the principle of alternative possibilities.
Namely, $\B_C\phi$ stands for ``statement $\phi$ is true and coalition (a set of agents)  $C$ had a strategy to prevent it''. Thus, our modalities $\A\phi$ and $\D\phi$ correspond to their modalities $\B_{\{\mbox{\scriptsize attacker}\}}\phi$ and $\B_{\{\mbox{\scriptsize defender}\}}\phi$. 
In spite of this {\em syntactic similarity} between their and our works, the resulting axiomatic systems are quite different, which comes from the {\em semantic difference} between strategic games and security games. In security games, the attacker knows the defender's strategy while in a similar strategic game she would not. There are three aspects in which this work is different from~\cite{nt19aaai}: 
\begin{enumerate}
    \item As stated above, in security games modality $\D$ is expressible through modalities $\A$ and $\N$, while in strategic games modality $\B_{\{\mbox{\scriptsize defender}\}}$ is not expressible through modalities $ \B_{\{\mbox{\scriptsize attacker}\}}$ and $\N$.
    \item Two of our core axioms for modality $\A$, the Conjunction axiom and the No Blame axiom capture the asymmetry of information in security games. They are not sound in strategic games. The Fairness axiom from~\cite{nt19aaai} is not sound in our setting. We further discuss this in the Axioms section.
    \item The proof of the completeness is using a completely different construction from the one used in~\cite{nt19aaai}. This is discussed in section Completeness. 
\end{enumerate}    


\section{Syntax and Semantics}\label{syntax and semantics section}

In this paper we consider a fixed set of propositional variables $\sf Prop$. 
The language $\Phi$ of our logical system is defined by the grammar:
$
\phi := p\;|\;\neg\phi\;|\;\phi\to\phi\;|\;\N\phi\;|\; \A\phi.
$

As usual, we assume that connectives $\wedge$, $\vee$, and $\leftrightarrow$ are defined through connectives $\to$ and $\neg$ in the standard way. Next, we formally define security games (or just ``games'').

\begin{definition}\label{game definition}
A game is a tuple $(\mathcal{D},\{\mathcal{A}_d\}_{d\in\mathcal{D}},\pi)$, where
\begin{enumerate}
    \item set $\mathcal{D}$ is a set of actions of the defender,
    \item non-empty set $\mathcal{A}_d$ is a set of actions of the attacker in response to the action $d\in\mathcal{D}$ of the defender,
    \item valuation $\pi(p)$ of a propositional variable $p$ is an arbitrary set of pairs $(d,a)$ such that $d\in\mathcal{D}$ and $a\in\mathcal{A}_d$.
\end{enumerate}
\end{definition} 

In game $G_1$ from the introduction, the set of actions $\mathcal{D}$ of the defender is a three-element set $\{75/25,50/50,25/75\}$. For each action $d\in\mathcal{D}$ of the defender in this game, the set of responses $\mathcal{A}_d$ is the same two-element set $\{T1,T2\}$. Informally, $\pi(p)$ describes the set of action profiles $(d,a)$ under which statement $p$ is true.

The next definition is the core definition of our paper. Its item 5 defines blameworthiness of the attacker in security games using the principle of alternative possibilities~\cite{f69tjop,w17}: the attacker is blamable for statement $\phi$ under action profile $(d,a)$ if $\phi$ is true under this profile and the attacker had an opportunity to prevent $\phi$.

\begin{definition}\label{sat} 
For any action $d\in\mathcal{D}$ of the defender and any response action $a\in\mathcal{A}_d$ of the attacker in a game $(\mathcal{D},\{\mathcal{A}_d\}_{d\in\mathcal{D}},\pi)$ and any formula $\phi\in\Phi$, the satisfiability relation $(d,a)\Vdash\phi$ is defined recursively as follows:
\begin{enumerate}
    \item $(d,a)\Vdash p$ if $(d,a)\in \pi(p)$, where $p\in {\sf Prop}$,
    \item $(d,a)\Vdash \neg\phi$ if $(d,a)\nVdash \phi$,
    \item $(d,a)\Vdash\phi\to\psi$ if $(d,a)\nVdash\phi$ or $(d,a)\Vdash\psi$,
    \item $(d,a)\Vdash\N\phi$ if $(d',a')\Vdash\phi$ for each action $d'\in\mathcal{D}$ of the defender and each response action $a'\in\mathcal{A}_{d'}$ of the attacker,
    \item $(d,a)\Vdash\A\phi$ if $(d,a)\Vdash\phi$ and there is a response action $a'\in\mathcal{A}_d$ of the attacker such that  $(d,a')\nVdash\phi$.
\end{enumerate}
\end{definition}

As defined above, language $\Phi$ includes the attacker's blameworthiness modality $\A$, but does not include the defender's blameworthiness modality $\D$. If modality $\D$ is added to language $\Phi$ to form language $\Phi^+$, then Definition~\ref{sat} would need to be extended by an additional item:
\begin{enumerate}
{\em 
  \setcounter{enumi}{5}
  \item $(d,a)\Vdash\D\phi$ if $(d,a)\Vdash\phi$ and there is an action $d'\in\mathcal{D}$ of the defender such that for each response action $a'\in \mathcal{A}_{d'}$ of the attacker, $(d',a')\nVdash\phi$.}
\end{enumerate}
As mentioned in the introduction, we do not include modality $\D$ into language $\Phi$ because it is expressible through modalities $\A$ and $\N$. Indeed, the following lemma holds for any formula $\phi\in\Phi^+$: 

\begin{lemma}\label{D through A lemma}
$(d,a)\Vdash \D\phi$ iff  $(d,a)\Vdash \phi\wedge \neg\N(\neg\phi\to\A\neg\phi)$.
\end{lemma}
\begin{proof}
$(\Rightarrow):$ Suppose that $(d,a)\nVdash \phi\wedge \neg\N(\neg\phi\to\A\neg\phi)$. Thus, either $(d,a)\nVdash \phi$ or $(d,a) \Vdash\N(\neg\phi\to\A\neg\phi)$. In the first case, $(d,a)\nVdash\D\phi$ by item 6 above.

Next assume that $(d,a) \Vdash\N(\neg\phi\to\A\neg\phi)$. By item 6, to prove $(d,a)\nVdash \D\phi$, it suffices to show that for any action $d'\in\mathcal{D}$ of the defender there is a response action $a'\in\mathcal{A}_{d'}$ of the attacker, such that $(d',a')\Vdash\phi$. Indeed, consider any action $d'\in\mathcal{D}$ of the defender. By Definition~\ref{game definition}, set $\mathcal{A}_{d'}$ is not empty. Let $a_1\in \mathcal{A}_{d'}$ be an arbitrary response action of the attacker  on action $d'$. Assumption $(d,a) \Vdash\N(\neg\phi\to\A\neg\phi)$, by item 4 of Definition~\ref{sat}, implies $(d',a_1)\Vdash\neg\phi\to\A\neg\phi$. We consider the following two cases separately:

\noindent{\bf Case I:} $(d',a_1)\Vdash\phi$. Then, choose the response action $a'$ to be $a_1$ to have $(d',a')\Vdash\phi$.

\noindent{\bf Case II:} $(d',a_1)\nVdash\phi$. Thus, $(d',a_1)\Vdash\neg\phi$ by item 2 of Definition~\ref{sat}. Hence, $(d',a_1)\Vdash\A\neg\phi$ by item 3 of Definition~\ref{sat} because $(d',a_1)\Vdash\neg\phi\to\A\neg\phi$. Thus, by item 5 of Definition~\ref{sat}, there is a response action $a_2\in\mathcal{A}_{d'}$ of the attacker such that $(d',a_2)\nVdash\neg\phi$. Hence, $(d',a_2)\Vdash\phi$ by item 2 of Definition~\ref{sat}. Then, choose the response action $a'$ to be $a_2$ to have $(d',a')\Vdash\phi$.

\noindent$(\Leftarrow):$ Suppose that $(d,a) \Vdash \phi\wedge \neg\N(\neg\phi\to\A\neg\phi)$. Thus, 
\begin{equation}\label{DA eq}
    (d,a) \Vdash \phi
\end{equation}
and  $(d,a) \nVdash\N(\neg\phi\to\A\neg\phi)$. The latter, by item 4 of Definition~\ref{sat}, implies that there is an action $d'\in\mathcal{D}$ of the defender and a response action $a'\in\mathcal{A}_{d'}$ of the attacker such that $(d',a') \nVdash\neg\phi\to\A\neg\phi$. Thus, $(d',a') \Vdash\neg\phi$ and $(d',a') \nVdash\A\neg\phi$ by item 3 of Definition~\ref{sat}. Then, $(d',a'') \Vdash\neg\phi$ for each response action $a''\in\mathcal{A}_{d'}$ of the attacker, by item 5 of Definition~\ref{sat}. 
Thus, $(d',a'') \nVdash\phi$ for each response action $a''\in\mathcal{A}_{d'}$ of the attacker, by item 2 of Definition~\ref{sat}. 
Hence, there exists an action $d'\in\mathcal{D}$ of the defender such that $(d',a'') \nVdash\phi$ for each response action $a''\in\mathcal{A}_{d'}$ of the attacker.
Therefore, statement~(\ref{DA eq}) implies $(d,a) \Vdash \D\phi$ by item 6 above.
\end{proof}

\section{Axioms}\label{axioms section}

In addition to the propositional tautologies in  language $\Phi$, our logical system contains the following axioms.

\begin{enumerate}
\item Truth: $\Box\phi\to\phi$, where $\Box\in\{\N,\A\}$,
\item Negative Introspection: $\neg\N\phi\to\N\neg\N\phi$,
\item Distributivity: 
$\N(\phi\to\psi)\to(\N\phi\to\N\psi)$,
\item Unavoidability: $\N\phi\to\neg\A\phi$,
\item Strict Conditional:  $\N(\phi\to\psi)\to(\A\psi\to(\phi\to \A\phi))$,
\item Conjunction: $\A(\phi\wedge\psi)\to(\A\phi\vee\A\psi)$,
\item No Blame: $\neg\A(\phi\to\A\phi)$.
\end{enumerate}

The Truth (for $\N$), the Negative Introspection, and the Distributivity axioms are the well known S5 properties of the necessity modality $\N$. The Truth axiom for modality $\A$ states that the attacker can only be blamed for something true. The Unavoidability axiom states that the attacker cannot be blamed for something that could not be prevented. 

The Strict Conditional axiom states that if statement $\psi$ is true under each action profile where $\phi$ is true, the attacker is blameable for $\psi$, and $\phi$ is true, then the attacker is also blamable for $\phi$. Indeed, because statement $\psi$ is true under each action profile where $\phi$ is true, any action of the attacker that prevents $\psi$ also prevents $\phi$. Hence, if the attacker is blameable for $\psi$ and $\phi$ is true, then the attacker is also blamable for $\phi$. We formalize this argument in Lemma~\ref{strict conditional soundness}.

The Truth axiom, the Unavoidability axiom, and the Strict Conditional axiom hold not only for modality $\A$, but for modality $\D$ as well. These axioms are also true for strategic games. 

The Conjunction and the No Blame axioms are the key axioms of our logical system. They capture the {\em asymmetry of information} in security games. Both of these axioms are true for the attacker's blameworthiness modality $\A$ -- their soundness is proven in the appendix. However, as Lemma~\ref{conjunction lemma axioms section} and Lemma~\ref{no blame lemma axioms section} show, they are not true for the defender's blameworthiness modality $\D$ in game $G_2$ depicted in Figure~\ref{conjunction-axiom figure}. Lemma~\ref{aux lemma axioms section} is an auxiliary statement about game $G_2$ used in the proofs of these two lemmas.

\begin{figure}[ht]
\begin{center}
\vspace{0mm}
\scalebox{0.48}{\includegraphics{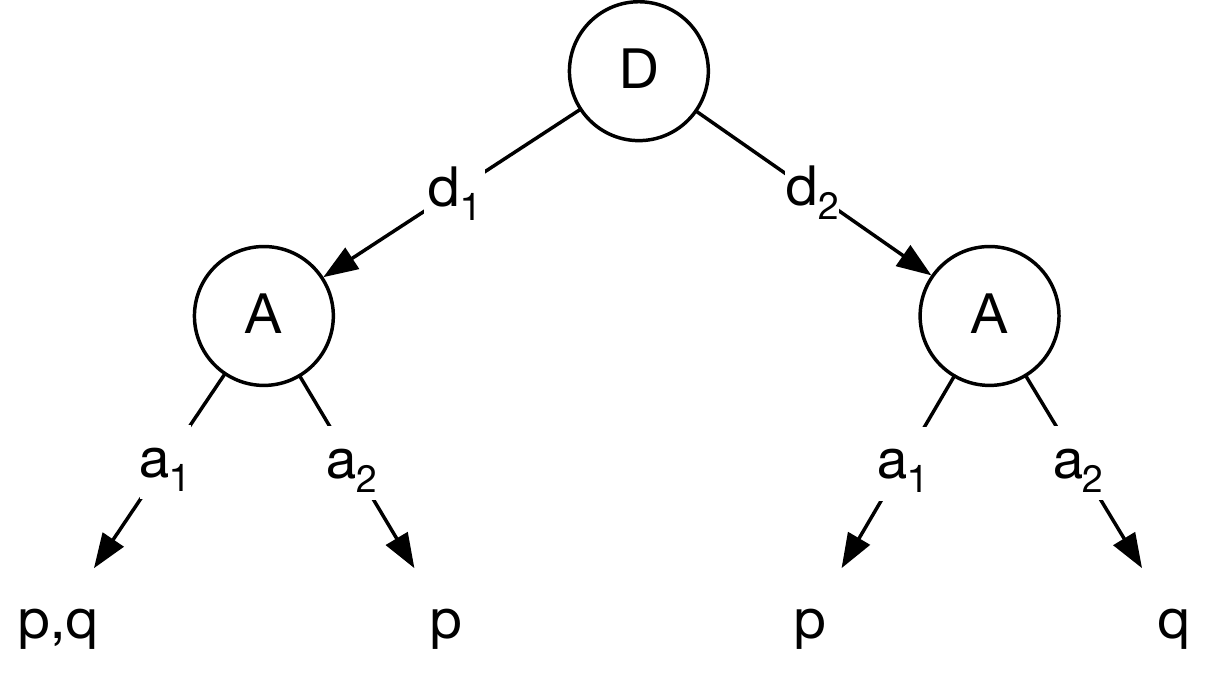}}
\vspace{0mm}
\caption{Game $G_2$, where $(d_1,a_1)\nVdash \D(p\wedge q)\to(\D p\vee\D q)$, $(d_2,a_2)\Vdash \D(p\to\D p)$, and $(d_2,a_1)\nVdash \A p \to \N(p\to \A p)$.}\label{conjunction-axiom figure}
\end{center}
\vspace{-3mm}
\end{figure}

\begin{lemma}\label{aux lemma axioms section}
$(d,a)\nVdash\D p$ and $(d,a)\nVdash\D q$ for each action $d$ of the defender and each response action $a$ of the attacker in game $G_2$.
\end{lemma}
\begin{proof}
Note that  $(d_1,a_1)\Vdash p$ and $(d_2,a_1)\Vdash p$, see Figure~\ref{conjunction-axiom figure}. Thus, for each action $d'$ of the defender there is an action $a'$ of the attacker such that $(d',a')\Vdash p$. Hence, $(d,a)\nVdash \D p$ by item 6 after Definition~\ref{sat}. Similarly, $(d_1,a_1)\Vdash q$ and $(d_2,a_2)\Vdash q$ imply that $(d,a)\nVdash \D q$.
\end{proof}

\begin{lemma}\label{conjunction lemma axioms section}
$(d_1,a_1)\nVdash \D(p\wedge q)\to(\D p\vee\D q)$.
\end{lemma}
\begin{proof}
By Lemma~\ref{aux lemma axioms section}, it suffices to show that $(d_1,a_1)\Vdash \D(p\wedge q)$. Indeed, observe that $(d_2,a_1)\nVdash p\wedge q$ and $(d_2,a_2)\nVdash p\wedge q$, see Figure~\ref{conjunction-axiom figure}. Thus, $(d_2,a)\nVdash p\wedge q$ for each response action $a$ of the attacker on action $d_2$ of the defender. Also, $(d_1,a_1)\Vdash p\wedge q$, see Figure~\ref{conjunction-axiom figure}. Therefore, $(d_1,a_1)\Vdash \D(p\wedge q)$ by item 6 after Definition~\ref{sat}.
\end{proof}

\begin{lemma}\label{no blame lemma axioms section}
$(d_2,a_2)\Vdash \D(p\to\D p)$.
\end{lemma}
\begin{proof}
$(d_1,a_1)\nVdash \D p$ and $(d_1,a_2)\nVdash \D p$ by Lemma~\ref{aux lemma axioms section}. Thus, $(d_1,a_1)\nVdash p\to \D p$ and $(d_1,a_2)\nVdash p\to \D p$ by item 3 of Definition~\ref{sat} and because $(d_1,a_1)\Vdash p$ and $(d_1,a_2)\Vdash p$, see Figure~\ref{conjunction-axiom figure}. Thus, $(d_1,a)\nVdash p\to \D p$ for each response action $a$ of the attacker on action $d$ of the defender. At the same time, $(d_2,a_2)\Vdash p\to\D p$ by item 3 of Definition~\ref{sat} because $(d_2,a_2)\nVdash p$, see Figure~\ref{conjunction-axiom figure}. Therefore, $(d_2,a_2)\Vdash \D(p\to\D p)$ by item 6 after Definition~\ref{sat}.
\end{proof}

Informally, the Conjunction and the No Blame axioms capture the properties of the asymmetry of the information in security games and thus they cannot be true in strategic games~\cite{nt19aaai} where the information is symmetric. A strategic game in which these axioms fail could be constructed by modifying the security game $G_2$ into a strategic game. 

The logical system for blameworthiness in strategic games~\cite{nt19aaai} includes the Fairness axiom: $\B_C\phi\to\N(\phi\to\B_C\phi)$. In the next two lemmas we show that in the case of security games this axiom is not sound for modality $\A$, but is sound for modality $\D$.

\begin{lemma}
$(d_2,a_1)\nVdash \A p \to \N(p\to \A p)$ in game $G_2$.
\end{lemma}
\begin{proof}
Note that $(d_2,a_1)\Vdash p$ and $(d_2,a_2)\nVdash p$, see Figure~\ref{conjunction-axiom figure}. Thus, $(d_2,a_1)\Vdash \A p$ by item 5 of Definition~\ref{sat}. Suppose that $(d_2,a_1)\Vdash \A p \to \N(p\to \A p)$. Hence, $(d_2,a_1)\Vdash \N(p\to \A p)$ by item 3 of Definition~\ref{sat}. Thus, $(d_1,a_1)\Vdash p\to \A p$ by item 4 of Definition~\ref{sat}. Note that $(d_1,a_1)\Vdash p$, see Figure~\ref{conjunction-axiom figure}. Hence, $(d_1,a_1)\Vdash \A p$ by item 3 of Definition~\ref{sat}. Then, by item 5 of Definition~\ref{sat}, there must exists a response action $a'\in\mathcal{D}_{d_1}$ of the attacker such that  $(d_1,a')\nVdash p$. However, such an action $a'$ does not exist because $(d_1,a_1)\Vdash p$ and $(d_1,a_2)\Vdash p$, see Figure~\ref{conjunction-axiom figure}.
\end{proof}

\begin{lemma}
$(d,a)\Vdash \D\phi\to\N(\phi\to\D\phi)$ for any formula $\phi\in \Phi^+$, any defender's action $d\in\mathcal{D}$, and any attacker's response action $a\in \mathcal{A}_d$ in an arbitrary security game $(\mathcal{D},\{\mathcal{A}_d\}_{d\in\mathcal{D}},\pi)$.
\end{lemma}
\begin{proof}
Suppose that $(d,a)\nVdash \D\phi\to\N(\phi\to\D\phi)$. Thus, $(d,a)\Vdash \D\phi$ and $(d,a)\nVdash \N(\phi\to\D\phi)$ by item 3 of Definition~\ref{sat}. By item 6 after Definition~\ref{sat}, statement $(d,a)\Vdash \D\phi$, implies that $(d,a)\Vdash\phi$.

By item 4 of  Definition~\ref{sat}, statement $(d,a)\nVdash \N(\phi\to\D\phi)$ implies that there is an action $d_1\in\mathcal{D}$ of the defender and a response action $a_1\in\mathcal{A}_{d_1}$ of the attacker such that $(d_1,a_1)\nVdash \phi\to\D\phi$. Thus, $(d_1,a_1)\Vdash \phi$ and $(d_1,a_1)\nVdash \D\phi$ by item 3 of  Definition~\ref{sat}. Hence, by item 6 after Definition~\ref{sat}, for each action $d'\in\mathcal{D}$ of the defender there is a response action $a'\in\mathcal{A}_{d'}$ of the attacker such that $(d',a')\Vdash \phi$. Then, $(d,a)\nVdash \D\phi$ by item 6 after Definition~\ref{sat} because $(d,a)\Vdash \phi$, which is a contradiction.
\end{proof}

We write $\vdash\phi$ if formula $\phi$ is provable from the axioms of our system using the Modus Ponens and
the Necessitation inference rules:
$$
\dfrac{\phi,\;\;\;\phi\to\psi}{\psi},
\hspace{20mm}
\dfrac{\phi}{\N\phi}.
$$
We write $X\vdash\phi$ if formula $\phi$ is provable from the theorems of our logical system and an additional set of axioms $X$ using only the Modus Ponens inference rule.

We conclude this section with an example of a formal proof in our logical system. The lemma below is used later in the proof of the completeness.

\begin{lemma}\label{biconditional lemma}
If $\vdash\phi\leftrightarrow\psi$, then $\vdash\A\phi\to\A\psi$.
\end{lemma}
\begin{proof}
By the Strict Conditional axiom,
$$
\vdash \N(\psi\to\phi)\to(\A\phi\to(\psi\to \A\psi)).
$$
Assumption  $\vdash \phi\leftrightarrow \psi$ implies $\vdash \psi\to \phi$ by the laws of propositional reasoning. Thus, $\vdash \N(\psi\to \phi)$ by the Necessitation inference rule. Hence, by the Modus Ponens rule,
$$
\vdash \A\phi\to(\psi\to \A\psi).
$$
Thus, by the laws of propositional reasoning,
\begin{equation}\label{sofia}
\vdash (\A\phi\to\psi)\to (\A\phi\to \A\psi).
\end{equation}
Note that $\vdash \A\phi\to\phi$ by the Truth axiom. At the same time, $\vdash \phi\leftrightarrow \psi$ by the assumption of the lemma. Thus, by the laws of propositional reasoning, $\vdash \A\phi\to\psi$. Therefore,
$
\vdash \A\phi\to \A\psi
$
by the Modus Ponens inference rule from statement~(\ref{sofia}).
\end{proof}

\section{Soundness}

In this section we prove the soundness of our logical system. The soundness of the Truth, the Negative Introspection, and the Distributivity axioms and of the two inference rules is straightforward. Below we prove the soundness of each of the remaining axioms as a separate lemma for any action $d\in\mathcal{D}$ of the defender, any response action $a\in\mathcal{A}_d$ of the attacker of an arbitrary security game $(\mathcal{D},\{\mathcal{A}_d\}_{d\in\mathcal{D}},\pi)$ and any formulae $\phi,\psi\in\Phi$. 

\begin{lemma}
If $(d,a)\Vdash \N\phi$, then $(d,a)\nVdash\A\phi$.
\end{lemma}
\begin{proof}
By item 4 of Definition~\ref{sat}, the assumption $(d,a)\Vdash \N\phi$ implies that $(d',a')\Vdash\phi$ for each action $d'\in\mathcal{D}$ of the defender and each response action $a'\in\mathcal{A}_{d'}$ of the attacker. In particular, $(d,a')\Vdash\phi$ for each response action $a'\in\mathcal{A}_{d}$ of the attacker. Therefore, $(d,a)\nVdash\A\phi$ by item 5 of Definition~\ref{sat}.
\end{proof}

\begin{lemma}\label{conjunction sound}
If $(d,a)\Vdash \A(\phi\wedge\psi)$, then either $(d,a)\Vdash \A\phi$ or $(d,a)\Vdash \A\psi$.
\end{lemma}
\begin{proof}
By item 5 of Definition~\ref{sat}, the assumption $(d,a)\Vdash \A(\phi\wedge\psi)$ implies that $(d,a)\Vdash \phi\wedge\psi$ and there is a response action $a'\in\mathcal{A}_d$ of the attacker such that $(d,a')\nVdash \phi\wedge\psi$. Hence, either $(d,a')\nVdash \phi$ or $(d,a')\nVdash \psi$. Without loss of generality, suppose that $(d,a')\nVdash \phi$. At the same time, statement $(d,a)\Vdash \phi\wedge\psi$ implies that $(d,a)\Vdash \phi$. Hence, $(d,a)\Vdash \phi$ and $(d,a')\nVdash \phi$. Therefore, $(d,a)\Vdash \A\phi$ by item 5 of Definition~\ref{sat}.
\end{proof}

\begin{lemma}
$(d,a)\nVdash\A(\phi\to\A\phi)$.
\end{lemma}
\begin{proof}
Suppose that $(d,a)\Vdash\A(\phi\to\A\phi)$. Thus, by item 5 of Definition~\ref{sat}, 
\begin{equation}\label{no blame sound eq}
    (d,a)\Vdash\phi\to\A\phi
\end{equation}
and there is a response action $a'\in\mathcal{A}_d$ of the attacker such that $(d,a')\nVdash\phi\to\A\phi$. Hence, $(d,a')\Vdash\phi$ and $(d,a')\nVdash\A\phi$ by item 3 of Definition~\ref{sat}. Thus, 
\begin{equation}\label{no blame sound eq 2}
    (d,a'')\Vdash\phi
\end{equation}
for any response action $a''\in\mathcal{A}_d$ of the attacker, by item 5 of Definition~\ref{sat}. In particular, $(d,a)\Vdash\phi$. Then, $(d,a)\Vdash\A\phi$ due to statement~(\ref{no blame sound eq}) and item 3 of Definition~\ref{sat}. Thus, by item 5 of Definition~\ref{sat}, there must exist a response action $b\in\mathcal{A}_d$ of the attacker such that $(d,b)\nVdash\phi$, which contradicts to statement~(\ref{no blame sound eq 2}).
\end{proof}

\begin{lemma}\label{strict conditional soundness}
If $(d,a)\Vdash \N(\phi\to\psi)$, $(d,a)\Vdash \A\psi$, and $(d,a)\Vdash \phi$, then $(d,a)\Vdash \A\phi$.
\end{lemma}
\begin{proof}
By item 5 of Definition~\ref{sat}, the assumption $(d,a)\Vdash \A\psi$ implies that there is a response action $a'\in\mathcal{A}_d$ of the attacker such that $(d,a')\nVdash \psi$. At the same time, $(d,a')\Vdash \phi\to\psi$ by item 4 of Definition~\ref{sat} and the assumption $(d,a)\Vdash \N(\phi\to\psi)$. Hence, $(d,a')\nVdash \phi$ by item 3 of Definition~\ref{sat}. Therefore, $(d,a)\Vdash \A\phi$ by the assumption $(d,a)\Vdash \phi$ and item 5 of Definition~\ref{sat}.
\end{proof}

\section{Completeness}

In this section we prove the completeness of our logical system in three steps. First, we introduce an auxiliary modality $\R$ as an abbreviation definable through modality $\A$. Next, we define a canonical security game and prove its basic property. Finally, we state and prove the strong completeness theorem for our logical system.

\subsection{Preliminaries}

Let $\R\phi$ be an abbreviation for $\neg(\phi\to\A\phi)$. Note that $\R\phi$ stands for ``statement $\phi$ is true, but the attacker cannot be blamed for it''. In other words, $\R\phi$ means that {\em the defender's} action {\em unavoidably} led to $\phi$ being true. This modality is not present in~\cite{nt19aaai}. In the context of STIT logic, but not in the context of security games, a similar single-agent modality was studied in \cite{x98jpl}. The same modality for coalitions was investigated in \cite{bht09jancl}.  Below we prove the key properties of modality $\R$ that are used later in the proof of the completeness.

\begin{lemma}\label{N to R}
$\vdash\N\phi\to\R\phi$.
\end{lemma}
\begin{proof}
By the Unavoidability axiom, $\vdash \N\phi\to\neg\A\phi$. At the same time, $\vdash \N\phi\to\phi$ by the Truth axiom. Hence, by propositional reasoning, $\vdash \N\phi\to\phi\wedge\neg\A\phi$. Thus, again by propositional reasoning, $\vdash \N\phi\to\neg(\phi\to\A\phi)$. Therefore, $\vdash\N\phi\to\R\phi$ by the definition of modality $\R$.
\end{proof}

The next four lemmas show that $\R$ is an S5 modality.

\begin{lemma}\label{necessitation rule for R}
Inference rule $\dfrac{\phi}{\R\phi}$ is derivable.
\end{lemma}
\begin{proof}
Suppose that $\vdash\phi$. Thus, $\vdash\N\phi$ by the Necessitation inference rule. Therefore, $\vdash\R\phi$ by Lemma~\ref{N to R} and the Modus Ponens inference rule.
\end{proof}

\begin{lemma}\label{truth axiom for R}
$\vdash\R\phi\to\phi$.
\end{lemma}
\begin{proof}
Note that formula $\neg(\phi\to\A\phi)\to\phi$ is a propositional tautology. Thus, $\vdash \R\phi\to\phi$ by the definition of the modality $\R$.
\end{proof}

\begin{lemma}\label{K axiom for R}
$\vdash\R(\phi\to\psi)\to(\R\phi\to\R\psi)$.
\end{lemma}
\begin{proof}
Note that the following formula is a propositional tautology
\begin{eqnarray*}
&&\neg((\phi\to\psi)\to\A(\phi\to\psi))\to\\
&&\hspace{10mm}(\neg(\phi\to\A\phi) \to (\neg\A(\phi\to\psi) \wedge \neg\A\phi)).
\end{eqnarray*}
Thus, it follows from the definition of the modality $\R$ that 
\begin{eqnarray*}
&&\vdash \R(\phi\to\psi)\to(\R\phi \to (\neg\A(\phi\to\psi) \wedge \neg\A\phi)).
\end{eqnarray*}
At the same time, formula 
\begin{eqnarray*}
&&(\neg\A(\phi\to\psi) \wedge \neg\A\phi)\to \neg\A((\phi\to\psi)\wedge\phi) 
\end{eqnarray*}
is a contrapositive of the Conjunction axiom. Thus, by the laws of propositional reasoning,
\begin{eqnarray}\label{alpha eq}
&&\vdash \R(\phi\to\psi)\to(\R\phi \to \neg\A((\phi\to\psi)\wedge\phi).
\end{eqnarray}
Next, note that the following formula is also a propositional tautology
$
((\phi\to\psi)\wedge\phi)\to\psi
$.
Hence, by the Necessitation inference rule,
$
\vdash\N(((\phi\to\psi)\wedge\phi)\to\psi)
$.
Thus, by the Strict Conditional axiom and the Modus Ponens inference rule,
$$
\vdash\A\psi\to  ((\phi\to\psi)\wedge\phi\to\A((\phi\to\psi)\wedge \phi)).
$$
Then, by the laws of propositional reasoning,
$$
\vdash\neg\A((\phi\to\psi)\wedge \phi)\to ((\phi\to\psi)\wedge \phi\to\neg\A\psi).
$$
Hence, by propositional reasoning using statement~(\ref{alpha eq}),
\begin{equation}\label{beta equation}
\vdash \R(\phi\to\psi)\to(\R\phi \to ((\phi\to\psi)\wedge \phi\to\neg\A\psi)).    
\end{equation}
Note that the following formula is a propositional tautology
\begin{eqnarray*}
&&\neg((\phi\to\psi)\to\A(\phi\to\psi))\to\\
&&\hspace{10mm}(\neg(\phi\to\A\phi) \to ((\phi\to\psi)\wedge \phi)).
\end{eqnarray*}
Thus, it follows from the definition of the modality $\R$ that
\begin{equation}\label{gamma equation}
\vdash\R(\phi\to\psi)\to(\R\phi \to ((\phi\to\psi)\wedge \phi)).
\end{equation}
Then, by  propositional reasoning using statement~(\ref{beta equation}),
\begin{equation}\label{delta equation}
\vdash \R(\phi\to\psi)\to(\R\phi \to \neg\A\psi).    
\end{equation}
Additionally, note that $((\phi\to\psi)\wedge \phi)\to\psi$ is a propositional tautology. Hence, statement~(\ref{gamma equation}) also implies
$$
\vdash\R(\phi\to\psi)\to(\R\phi \to \psi).
$$
Thus, by propositional reasoning using statement~(\ref{delta equation}),
$$
\vdash \R(\phi\to\psi)\to(\R\phi \to (\psi\wedge \neg\A\psi)).
$$
Again by propositional reasoning,
$$
\vdash \R(\phi\to\psi)\to(\R\phi \to \neg(\psi\to\A\psi)).
$$
Therefore, $\vdash\R(\phi\to\psi)\to(\R\phi\to\R\psi)$ by the definition of the modality $\R$.
\end{proof}

\begin{lemma}\label{negative introspection for R}
$\vdash\neg\R\phi\to\R\neg\R\phi$.
\end{lemma}
\begin{proof}
Note that $\neg\neg(\phi\to\A\phi)\leftrightarrow(\phi\to\A\phi)$ is a propositional tautology. Thus, $\vdash\A\neg\neg(\phi\to\A\phi)\to\A(\phi\to\A\phi)$ by Lemma~\ref{biconditional lemma}. Hence, $\vdash\neg\A(\phi\to\A\phi)\to \neg\A\neg\neg(\phi\to\A\phi)$ by contraposition. Then, $\vdash\neg\A\neg\neg(\phi\to\A\phi)$ by the No Blame Axiom and the Modus Ponens inference rule.
Thus, by the laws of propositional reasoning,
$$
\vdash (\phi\to\A\phi) \to \neg((\phi\to\A\phi) \to \A\neg\neg(\phi\to\A\phi)).
$$
Hence, again by the laws of propositional reasoning,
$$
\vdash \neg\neg(\phi\to\A\phi) \to \neg(\neg\neg(\phi\to\A\phi) \to \A\neg\neg(\phi\to\A\phi)).
$$
Recall that $\R\phi$ is an abbreviation for $\neg(\phi\to\A\phi)$. Then,
$$
\vdash \neg\R\phi \to \neg(\neg\R\phi \to \A\neg\R\phi).
$$
Thus, $\vdash\neg\R\phi\to\R\neg\R\phi$ again by the definition of $\R$.
\end{proof}

The next two lemmas capture well known properties of S5 modalities. 

\begin{lemma}\label{superdistributivity}
If $\phi_1,\dots,\phi_n\vdash \psi$, then $\Box\phi_1,\dots,\Box\phi_n\vdash \Box\psi$, where $\Box$ is either modality $\N$ or modality $\R$.
\end{lemma}
\begin{proof}
First, consider the case when $\Box$ is modality $\N$. Assumption $\phi_1,\dots,\phi_n\vdash \psi$ by the deduction lemma implies that
$\vdash\phi_1\to(\phi_2\to\dots(\phi_n\to \psi)\dots)$. Hence, by the Necessitation rule, 
$\vdash\N(\phi_1\to(\phi_2\to\dots(\phi_n\to \psi)\dots))$.
Thus, by the Distributivity axiom and the Modus Ponens inference rule,
$\vdash\N\phi_1\to\N(\phi_2\to\dots(\phi_n\to \psi)\dots)$.
Hence, 
$\N\phi_1\vdash\N(\phi_2\to\dots(\phi_n\to \psi)\dots)$
again by the Modus Ponens inference rule.
By repeating the previous two steps $(n-1)$ more times,
$\N\phi_1,\dots,\N\phi_n\vdash \N\psi$.

The case when $\Box$ is modality $\R$ is similar, but it uses Lemma~\ref{necessitation rule for R} instead of the Necessitation inference rule and Lemma~\ref{K axiom for R} instead of the Distributivity axiom.
\end{proof}

\begin{lemma}\label{positive introspection lemma}
$\vdash\Box\phi\to\Box\Box\phi$ where $\Box$ is either modality $\N$ or modality $\R$. 
\end{lemma}
\begin{proof}
We first consider the case when $\Box$ is modality $\N$. 
Formula $\N\neg\N\phi\to\neg\N\phi$ is an instance of the Truth axiom. Thus, $\vdash \N\phi\to\neg\N\neg\N\phi$ by contraposition. Hence, taking into account the following instance of  the Negative Introspection axiom: $\neg\N\neg\N\phi\to\N\neg\N\neg\N\phi$,
we have 
\begin{equation}\label{pos intro eq 2}
\vdash \N\phi\to\N\neg\N\neg\N\phi.
\end{equation}

At the same time, $\neg\N\phi\to\N\neg\N\phi$ is an instance of the Negative Introspection axiom. Thus, $\vdash \neg\N\neg\N\phi\to \N\phi$ by the law of contrapositive in the propositional logic. Hence, by the Necessitation inference rule, 
$\vdash \N(\neg\N\neg\N\phi\to \N\phi)$. Thus, by  the Distributivity axiom and the Modus Ponens inference rule, 
$
  \vdash \N\neg\N\neg\N\phi\to \N\N\phi.
$
 The latter, together with statement~(\ref{pos intro eq 2}), implies the statement of the lemma by propositional reasoning.
 
The case when $\Box$ is modality $\R$ is similar, but it uses Lemma~\ref{truth axiom for R} instead of the Truth axiom, Lemma~\ref{negative introspection for R} instead of the Negative Introspection axiom, Lemma~\ref{necessitation rule for R} instead of the Necessitation inference rule, and Lemma~\ref{K axiom for R} instead of the Distributivity axiom. 
\end{proof}

\subsection{Canonical Security Game}

We define the canonical game $G(X)=(\Omega,\{\mathcal{A}_\delta\}_{\delta\in\Omega},\pi)$ for each maximal consistent set of formulae $X$.

\begin{definition}\label{canonical Omega}
$\Omega$ is the set of all maximal consistent sets of formulae such that if $\omega\in\Omega$, then $\{\phi\in\Phi\;|\;\N\phi\in X\}\subseteq \omega$.
\end{definition}

\begin{definition}\label{canonical sim}
$\omega\sim\omega'$ if $\forall\phi\in\Phi\;(\R\phi\in\omega\Leftrightarrow\R\phi\in\omega')$.
\end{definition}

Note that $\sim$ is an equivalence relation on set $\Omega$. The set $\mathcal{A}_\delta$ of possible responses by the  attacker on an action $\delta\in\Omega$ of the defender is the (nonempty) equivalence class of element $\delta$ with respect to this equivalence relation:

\begin{definition}\label{canonical A}
$\mathcal{A}_\delta=[\delta]$.
\end{definition}

Thus, each defender's action $\delta\in \Omega$ and each attacker's responses $\omega\in [\delta]$ are maximal consistent sets of formulae. This is significantly different from~\cite{nt19aaai}, where actions of all agents are formulae.

\begin{definition}\label{canonical pi}
$\pi(p)=\{(\delta,\omega)\in \Omega\times\Omega\;|\;\omega\in\mathcal{A}_\delta, p\in \omega\}$.
\end{definition}
This concludes the definition of the canonical game $G(X)$.

As usual, at the core of the proof of completeness is a truth lemma (or an induction lemma), which in our case is Lemma~\ref{truth lemma}. The next four lemmas are auxiliary statements used in the induction step of the proof of Lemma~\ref{truth lemma}.

\begin{lemma}\label{A child exists}
For any action $\delta\in\Omega$ of the defender, any response action $\omega\in [\delta]$ of the attacker, and any formula $\A\phi\in \omega$, we have
(i) $\phi\in\omega$ and
(ii) there is a response action $\omega'\in[\delta]$ such that $\phi\notin\omega'$.  
\end{lemma}
\begin{proof}
Assumption $\A\phi\in \omega$ implies that $\omega\vdash\phi$ by the Truth axiom and the Modus Ponens inference rule. Thus, $\phi\in\omega$ because set $\omega$ is maximal. This concludes the proof of the first statement.
To prove the second statement, consider the set of formulae
\begin{equation}\label{Y definition}
    Y=\{\neg\phi\}\cup\{\psi\;|\;\R\psi\in\omega\}\cup\{\chi\;|\;\N\chi\in \omega\}.
\end{equation}
\begin{claim}
Set $Y$ is consistent.
\end{claim}
\begin{proof-of-claim}
Suppose the opposite. Thus, there are
\begin{equation}\label{choice of psi and chi}
    \R\psi_1,\dots,\R\psi_k,\N\chi_1,\dots,\N\chi_n\in \omega
\end{equation}
such that
$
\psi_1,\dots,\psi_k,\chi_1,\dots,\chi_n\vdash \phi.
$
Hence, by Lemma~\ref{superdistributivity},
$
\R\psi_1,\dots,\R\psi_k,\R\chi_1,\dots,\R\chi_n\vdash \R\phi
$.
Then, by Lemma~\ref{N to R} and the Modus Ponens inference rule,
$
\R\psi_1,\dots,\R\psi_k,\N\chi_1,\dots,\N\chi_n\vdash \R\phi
$.
Thus, 
$
\omega\vdash \R\phi
$
by statement~(\ref{choice of psi and chi}).
Hence,
$
\omega\vdash \neg(\phi\to\A\phi)
$
by the definition of the modality $\R$. Then,
$
\omega\vdash \neg\A\phi
$
by the laws of the propositional reasoning, which contradicts the assumption $\A\phi\in \omega$ of the lemma because set $\omega$ is consistent. 
\end{proof-of-claim}
Let set $\omega'$ be any maximal consistent extension of set $Y$. Then, $\neg\phi\in \omega'$. Thus, $\phi\notin \omega'$ because set $\omega'$ is consistent.

\begin{claim}\label{omega' claim}
$\omega'\in \Omega$.
\end{claim}
\begin{proof-of-claim}
Consider any formula $\N\chi\in X$. By Definition~\ref{canonical Omega}, it suffices to show that $\chi\in\omega'$. Indeed, assumption  $\N\chi\in X$ implies that $X\vdash\N\N\chi$ by Lemma~\ref{positive introspection lemma}. Thus, $\N\N\chi\in X$ because set $X$ is maximal. Then, $\N\chi\in\omega$ by Definition~\ref{canonical Omega} and the assumption $\omega\in[\delta]\subseteq\Omega$ of the lemma. Hence, $\chi\in Y\subseteq \omega'$ by equation~(\ref{Y definition}) and the choice of set $\omega'$.
\end{proof-of-claim}

\begin{claim}
$\omega'\in[\delta]$.
\end{claim}
\begin{proof-of-claim}
Recall that $\omega\in[\delta]$ by the assumption of the lemma. Thus, by Claim~\ref{omega' claim}, it suffices to show that $\omega\sim\omega'$. 
Hence, by Definition~\ref{canonical sim}, it suffices to prove that $\R\psi\in \omega$ iff $\R\psi\in \omega'$ for each formula $\psi\in\Phi$. If $\R\psi\in \omega$, then $\omega\vdash\R\R\psi$ by Lemma~\ref{positive introspection lemma}. Hence, $\R\R\psi\in\omega$ because set $\omega$ is maximal.  Thus, $\R\psi\in Y\subseteq \omega'$ by equation~(\ref{Y definition}) and the choice of $\omega'$.

Suppose that $\R\psi\notin \omega$. Thus, $\omega\vdash \R\neg\R\psi$ by Lemma~\ref{negative introspection for R} and the Modus Ponens inference rule. Hence, $\R\neg\R\psi\in \omega$ because set $\omega$ is maximal.  Thus, $\neg\R\psi\in Y\subseteq \omega'$ by equation~(\ref{Y definition}) and the choice of set $\omega'$. Therefore, $\R\psi\notin \omega'$ because set $\omega'$ is consistent.
\end{proof-of-claim}
This concludes the proof of the lemma.
\end{proof}

\begin{lemma}\label{A child all}
For any action $\delta\in\Omega$ of the defender, any response action $\omega\in [\delta]$ of the attacker, and any formula $\phi\in\Phi$, if $\neg(\phi\to \A\phi)\in \omega$, then $\phi\in \omega'$ for each $\omega'\in[\delta]$.
\end{lemma}
\begin{proof}
Assumption $\neg(\phi\to \A\phi)\in \omega$ implies $\R\phi\in \omega$ by the definition of the modality $\R$. Note that $\omega\sim\omega'$ because $\omega,\omega'\in[\delta]$. Thus, $\R\phi\in \omega'$ by Definition~\ref{canonical sim}. Hence, $\omega'\vdash\phi$ by Lemma~\ref{truth axiom for R} and the Modus Ponens inference rule. Therefore, $\phi\in \omega'$ because set $\omega'$ is maximal.
\end{proof}

\begin{lemma}\label{N child all}
For any actions $\omega,\omega'\in\Omega$, if $\N\phi\in\omega$, then $\phi\in\omega'$. 
\end{lemma}
\begin{proof}
Suppose that $\phi\notin\omega'$.  Hence, $\N\phi\notin X$ by Definition~\ref{canonical Omega} and the assumption $\omega'\in\Omega$. Thus, $\neg\N\phi\in X$ because set $X$ is maximal. Then, $X\vdash\N\neg\N\phi$ by the Negative Introspection axiom and the Modus Ponens inference rule. Hence, $\N\neg\N\phi\in X$ again because set $X$ is maximal. Thus, $\neg\N\phi\in\omega$ by Definition~\ref{canonical Omega} and the assumption  $\omega\in\Omega$. Therefore, $\N\phi\notin \omega$ because set $\omega$ is consistent.
\end{proof}

\begin{lemma}\label{N child exists}
For any action $\omega\in\Omega$ and any formula $\neg\N\phi\in\omega$, there is an action $\omega'\in\Omega$ such that $\phi\notin\omega'$.
\end{lemma}
\begin{proof}
Consider the set of formulae
\begin{equation}\label{Y definition 2}
    Y = \{\neg\phi\}\cup\{\psi\;|\;\N\psi\in \omega\}.
\end{equation}
\begin{claim}
Set $Y$ is consistent.
\end{claim}
\begin{proof-of-claim}
Suppose the opposite. Thus, there are formulae
\begin{equation}\label{choice of psi}
    \N\psi_1,\dots,\N\psi_n\in \omega
\end{equation}
such that $\psi_1,\dots,\psi_n\vdash\phi$. Hence, $\N\psi_1,\dots,\N\psi_n\vdash\N\phi$ by Lemma~\ref{superdistributivity}. Thus, $\omega\vdash\N\phi$ by the assumption~(\ref{choice of psi}), which contradicts the assumption $\neg\N\phi\in\omega$ of the lemma because set $\omega$ is consistent.
\end{proof-of-claim}
Let set $\omega'$ be any maximal consistent extension of set $Y$. Then, $\neg\phi\in \omega'$. Thus, $\phi\notin \omega'$ because set $\omega'$ is consistent.

\begin{claim}
$\omega'\in \Omega$.
\end{claim}
\begin{proof-of-claim}
Consider any formula $\N\psi\in X$. By Definition~\ref{canonical Omega}, it suffices to show that $\psi\in\omega'$. Indeed, assumption  $\N\psi\in X$ implies that $X\vdash\N\N\psi$ by Lemma~\ref{positive introspection lemma}. Thus, $\N\N\psi\in X$ because set $X$ is maximal. Then, $\N\psi\in\omega$ by Definition~\ref{canonical Omega} and the assumption $\omega\in\Omega$ of the lemma. Therefore, $\psi\in Y\subseteq \omega'$ by equation~(\ref{Y definition 2}) and the choice of set $\omega'$.
\end{proof-of-claim}
This concludes the proof of the lemma.
\end{proof}

\begin{lemma}[truth lemma]\label{truth lemma}
For each formula $\phi$, each action of the defender $\delta\in\Omega$, and each response action $\omega \in [\delta]$ of the attacker,
$(\delta,\omega)\Vdash\phi$ iff $\phi\in\omega$.

\end{lemma}
\begin{proof}
We prove the lemma by structural induction on formula $\phi$. The case when formula $\phi$ is a propositional variable follows from Definition~\ref{sat} and Definition~\ref{canonical pi}. The cases when formula $\phi$ is a negation or an implication follow from Definition~\ref{sat} and the assumption of the maximality and the consistency of set $\omega$ in the standard way.

\vspace{2mm}
Suppose that formula $\phi$ has the form $\A\psi$.

\noindent$(\Rightarrow):$ Assume that $\A\psi\notin\omega$. Hence, $\omega\nvdash\A\psi$ because set $\omega$ is maximal. We consider the following two cases separately:

\noindent{\bf Case I:} $(\psi\to\A\psi)\in \omega$. Thus, statement $\omega\nvdash\A\psi$ implies $\omega\nvdash\psi$ by the contraposition of the Modus Ponens inference rule. Hence, $\psi\not\in \omega$. Then, $(\delta,\omega)\nVdash\phi$ by the induction hypothesis. Therefore, $(\delta,\omega)\nVdash\A\phi$ by item 5 of Definition~\ref{sat}.

\noindent{\bf Case II:} $(\psi\to\A\psi)\notin \omega$. Hence, $\neg(\psi\to\A\psi)\in \omega$ because set $\omega$ is maximal. Thus, $\psi\in\omega'$ for each action $\omega'\in [\delta]$, by Lemma~\ref{A child all}. Then, by the induction hypothesis, $(\delta,\omega')\Vdash\psi$ for each response action $\omega'\in[\delta]$ of the attacker on action $\delta\in\Omega$ of the defender. Therefore, $(\delta,\omega)\nVdash\A\psi$ by item 5 of  Definition~\ref{sat}.   

\noindent$(\Leftarrow):$ Assume that $\A\psi\in\omega$. Thus, by Lemma~\ref{A child exists}, we have
(i) $\psi\in\omega$ and
(ii) there is a response action $\omega'\in[\delta]$ such that $\psi\notin\omega'$.  Hence, by the induction hypothesis, 
(i) $(\delta,\omega)\Vdash\psi$ and
(ii) there is a response action $\omega'\in[\delta]$ of the attacker such that $(\delta,\omega')\nVdash \psi$. Therefore, $(\delta,\omega)\Vdash\A\psi$ by item 5 of Definition~\ref{sat}.

\vspace{2mm}
Next, assume formula $\phi$ has the form $\N\psi$. 

\noindent$(\Rightarrow):$  Let $\N\psi\notin \omega$. Thus, $\neg\N\psi\in\omega$ because set $\omega$ is maximal. Hence, by Lemma~\ref{N child exists}, there is an action $\omega'\in\Omega$ such that $\psi\notin\omega'$. Note that $\omega'\in[\omega']$ because $[\omega']$ is an equivalence class. Thus, $(\omega',\omega')\nVdash\psi$ by the induction hypothesis. Therefore, $(\delta,\omega)\nVdash\N\psi$ by item 4 of Definition~\ref{sat}.

\noindent$(\Leftarrow):$ Suppose that $\N\psi\in \omega$. Thus, $\psi\in \omega'$ for each action $\omega'\in\Omega$ by Lemma~\ref{N child all}. Hence, by the induction hypothesis, $(\delta',\omega')\Vdash\psi$ for each action $\delta'\in\Omega$ of the defender and each response action $\omega'\in[\delta']$ of the attacker. Therefore, $(\delta,\omega)\Vdash\N\psi$ by item 4 of Definition~\ref{sat}.
\end{proof}

Recall that the canonical game $G(X)$ is defined for an arbitrary maximal consistent set of formulae $X$.

\begin{lemma}\label{X lemma}
$X\in \Omega$.
\end{lemma}
\begin{proof}
Consider any formula $\N\phi\in X$. By Definition~\ref{canonical Omega}, it suffices to show that $\phi\in X$. Indeed, assumption $\N\phi\in X$ implies $X\vdash\phi$ by the Truth axiom and the Modus Ponens inference rule. Thus, $\phi\in X$ because set $X$ is maximal. 
\end{proof}

\subsection{Strong Completeness Theorem}

\begin{theorem}
If $X_0\nvdash\phi$, then there is an action $d\in\mathcal{D}$ of the defender and a response action $a\in\mathcal{A}_d$ of the attacker in a game $(\mathcal{D},\{\mathcal{A}_d\}_{d\in\mathcal{D}},\pi)$ such that $(d,a)\Vdash\chi$ for each formula $\chi\in X_0$ and $(d,a)\nVdash\phi$.
\end{theorem}
\begin{proof}
Let the set of formulae $X\subseteq\Phi$ be any maximal consistent extension of set $X_0\cup\{\neg\phi\}$. Then, $\phi\notin X$ because set $X$ is consistent.

Consider the canonical game $G(X)=(\Omega,\{\mathcal{A}_\delta\}_{\delta\in\Omega},\pi)$. Then, $X\in\Omega$ by Lemma~\ref{X lemma}. Also, $X\in[X]=\mathcal{A}_X$ because set $[X]$ is an equivalence class and because of Definition~\ref{canonical A}. Therefore, $(X,X)\Vdash\chi$ for each formula $\chi\in X_0\subseteq X$ and $(X,X)\nVdash\phi$ by Lemma~\ref{truth lemma}. 
\end{proof}

\section{Conclusion}

In this paper we gave a sound and complete axiomatic system that describes the properties of blameworthiness in security games. A natural next step is to generalize this work to arbitrary extensive form games. The Conjunction and the No Blame axioms in this paper are specific to security games and are not sound for arbitrary extensive form games. As we have seen in Lemma~\ref{conjunction lemma axioms section} and Lemma~\ref{no blame lemma axioms section}, these axioms are already not sound for the player who makes the first move in a security game. Although these axioms are sound for the player making the second move in security games, it is not sound for the second player in an arbitrary extensive form game. Consider, for example, game $G_3$ depicted in Figure~\ref{conclusion figure}. In this game, $(d_1,a_2)\Vdash \A(p\wedge q)$ because formula $p\wedge q$ is true under the action profile $(d_1,a_2)$, but the second player could have prevented it by using action $a_1$ instead of $a_2$. At the same time,  $(d_1,a_2)\nVdash \A p\vee  \A q$ because the second player has neither a strategy that would prevent $p$ nor a strategy that would prevent $q$. This is a counterexample for the Conjunction axiom.
\begin{figure}[ht]\label{conclusion figure}
\begin{center}
\vspace{0mm}
\scalebox{0.48}{\includegraphics{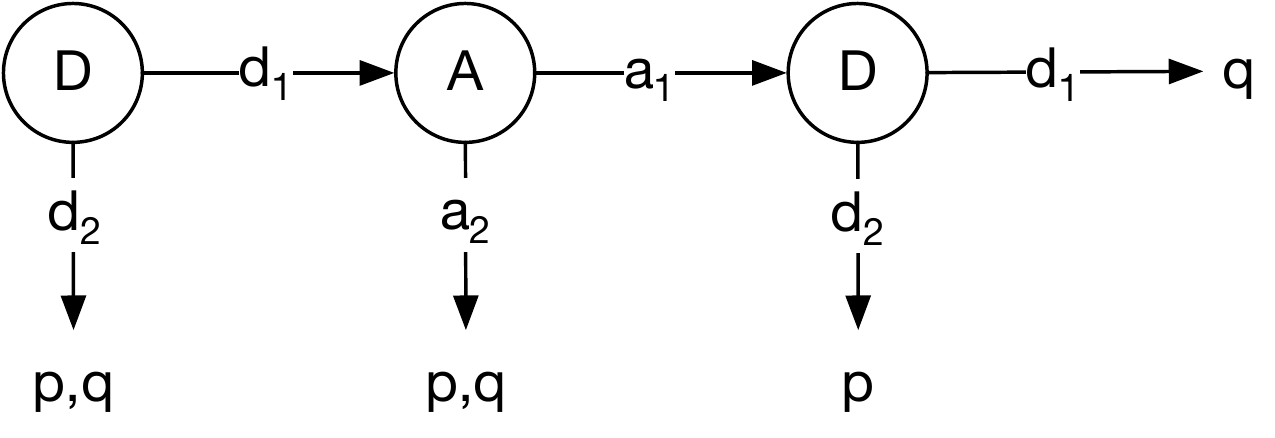}}
\vspace{0mm}
\caption{Game $G_3$, where $(d_1,a_2)\nVdash \A(p\wedge q)\to(\A p\vee\A q)$, and $(d_1,a_1,d_1)\Vdash \A(p\to\A p)$.
}\label{conclusion figure}
\end{center}
\end{figure}
The game $G_3$ also provides a counterexample for the No Blame axiom: $(d_1,a_1,d_1)\Vdash \A(p\to\A p)$. Indeed, $(d_1,a_1,d_1)\Vdash p\to\A p$ because $(d_1,a_1,d_1)\nVdash p$. At the same time, $(d_1,a_2)\nVdash p\to\A p$. Thus, the second player could have prevented $p\to\A p$ by using $a_2$ instead of $a_1$.

In addition to finding the right set of axioms, proving a completeness theorem would also require to recover the structure of the canonical game tree from a maximal consistent set of formulae. Finding the right set of axioms sound for all extensive form games and proving their completeness remains an open problem.


\bibliographystyle{aaai}
\bibliography{sp}

\end{document}